\documentclass{article}
\usepackage{spconf,amsmath,graphicx}

\usepackage{cite}
\usepackage{amsmath,amssymb,amsfonts,amsthm}
\usepackage{algorithmic}
\usepackage{graphicx}
\usepackage{textcomp}
\usepackage{xcolor}
\usepackage{hyperref}
\hypersetup{
    colorlinks=true,      
    linkcolor=blue,       
    citecolor=green,      
    urlcolor=cyan         
}

\def\x{{\mathbf x}}

\def\BibTeX{{\rm B\kern-.05em{\sc i\kern-.025em b}\kern-.08em
    T\kern-.1667em\lower.7ex\hbox{E}\kern-.125emX}}

\DeclareMathOperator*{\argmin}{arg\,min}
\def\Z{\mathcal{Z}}
\def\R{\mathcal{R}}
\def\x{\mathbf{x}}
\def\z{\mathbf{z}}
\def\q{\mathbf{q}}
\def\I{\Omega}
\newtheorem{thm}{Theorem}

\def\ie{{\it i.e.}}
\def\eg{{\it e.g.}}

\usepackage{xcolor}
\newcommand{\zf}[1]{\textcolor{black}{#1}}

\title{Tamper-evident Image using JPEG Fixed Points}
\name{Zhaofeng Si, Siwei Lyu}
\address{University at Buffalo, State University of New York, USA}
\begin{document}
%
\maketitle
\begin{abstract}
An intriguing phenomenon about JPEG compression has been observed since two decades ago- after repeating JPEG compression and decompression, it leads to a stable image that does not change anymore, which is a fixed point. In this work, we prove the existence of fixed points in the essential JPEG procedures. We analyze JPEG compression and decompression processes, revealing the existence of fixed points that can be reached within a few iterations. These fixed points are diverse and preserve the image's visual quality, ensuring minimal distortion. This result is used to develop a method to create a tamper-evident image from the original authentic image, which can expose tampering operations by showing deviations from the fixed point image.
\end{abstract}
%
\section{Introduction}

Since becoming an IEEE standard in 1992, the {\it Joint Photographic Experts Group} (JPEG) format \cite{wallace1992jpeg} has become the dominant image format due to its efficient compression and broad compatibility. JPEG employs a compression method that significantly reduces file size while maintaining visual quality. This is achieved by converting the image into the frequency domain using the Discrete Cosine Transform (DCT), quantizing the frequency coefficients to discard less noticeable details, and encoding the resulting data using Huffman coding. This combination of techniques makes JPEG particularly suitable for storing and sharing photographic images, striking a balance between compression efficiency and visual fidelity. It is estimated that over 70\% of all images are stored in the JPEG format \cite{jpeg2020usage}. 

An intriguing property of JPEG was first observed and elucidated by \cite{huang2010detecting} in the context of identifying features for detecting image manipulation. The study reveals that an image becomes unchanged after undergoing several rounds of the same JPEG compression and decompression process. In other words, if a single cycle of JPEG compression and decompression is considered a transformation of the image, referred to as a {\it JPEG transform}, then this transform exhibits the property of having {\it fixed points}, \ie, images that remain unaltered when the JPEG transform is applied. This phenomenon is visually illustrated in Fig.~\ref{fig:box_quality} using an example image.
\begin{figure}[t]
    \centering
    \includegraphics[width=1\linewidth]{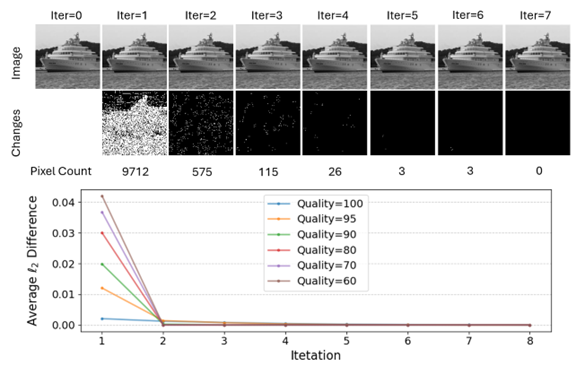}
    ~\vspace{-2em}
    \caption{\it \small A demonstration of JPEG fixed points is presented. {\bf (Top)}: An example image approaches its JPEG fixed point through repeated JPEG transforms, along with the locations and numbers of pixels that change in each iteration. {\bf (Bottom)}: The pixel-wise $\ell_2$ distance between images from consecutive iterations of JPEG transforms is plotted for the example image in the top row under varying JPEG compression qualities.}
    \label{fig:box_quality}
    ~\vspace{-2em}
\end{figure}

\begin{figure*}[t]
    \centering
    \includegraphics[width=1\linewidth]{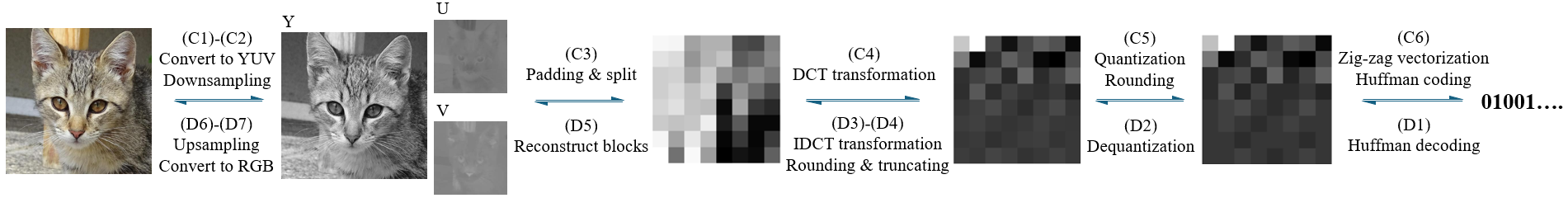}
    ~\vspace{-2em}
    \caption{\it \small The overall pipeline of JPEG compression and decompression.}
    \label{fig:jpeg_compression}
    ~\vspace{-2.5em}
\end{figure*}

In this work, we aim to conduct a rigorous investigation into the fixed point phenomenon of the JPEG transform. Specifically, we provide a mathematical formulation of the JPEG transform and prove the following: (a) the existence of fixed points and (b) that these fixed points can be reached within a finite number of steps starting from any arbitrary image. Using numerical experiments, we validate both the existence of JPEG fixed points and the convergence behavior of images toward them.

In addition, we leverage the JPEG fixed point phenomenon to develop a novel image integrity authentication scheme. This scheme enables the creation of {\it tamper-evident} images \cite{li2014image}, where any modifications to the image induce detectable changes, making tampering readily apparent. We demonstrate that tamper-evident images can be constructed from JPEG fixed points. Any alterations to the image will cause deviations from the JPEG fixed points, which can be detected as changes in the JPEG blocks after a single round of JPEG compression and decompression.

The proposed tamper-evident images based on JPEG fixed points have two advantages. Firstly,
tamper-evident images eliminate the need for external storage of verifiable features, as required by image fingerprinting schemes \cite{c2pa_specification}, or the embedding of hidden traces, as in image watermarking methods. The image itself serves as its proof of authenticity, making the scheme inherently self-evident. Secondly, since JPEG is a widely-used format and often the final step in the image processing pipeline, the proposed method is resilient to JPEG operations. This contrasts with the original approach \cite{li2014image} that may lose integrity traces due to JPEG.


\section{Fixed Points for the JPEG Transform}
\label{sec:main}

\subsection{Overview of JPEG}
\label{sec:jpeg}


The JPEG compression process The JPEG compression process described in \cite{wallace1992jpeg} consists of the following steps, as illustrated in Fig.~\ref{fig:jpeg_compression}:
~\vspace{-.5em}
\begin{itemize} \itemsep -.5em 
\item (C1) Convert the RGB image to a YUV image using RGB-YUV transform \cite{poynton2012digital}. \item (C2) Down-sample the chromatic channels (\ie, U and V channels) by a factor of two. \item (C3) Divide each channel into non-overlapping $8 \times 8$ blocks. If the image dimensions are not a multiple of eight, padding is applied to make them so. 
\item (C4) Perform a {\it discrete cosine transform} (DCT, Section \ref{sec:dct}) on the blocks. \item (C5) Quantize the DCT coefficients using a $8\times 8$ quantization table that contains quantization steps for each DCT channel. \item (C6) Apply Huffman coding to the quantized DCT coefficients in a zigzag order. 
~\vspace{-.75em}
\end{itemize}
The corresponding JPEG decompression process reverses these steps as follows:
~\vspace{-.75em}
\begin{itemize} \itemsep -.5em 
\item (D1) Use Huffman decoding to recover the quantized blocks. \item (D2) De-quantize the quantized blocks. \item (D3) Perform the {\it inverse DCT} (IDCT, Section \ref{sec:dct}) on the de-quantized blocks. \item (D4) Round the block values (introducing precision loss) and truncate them to the valid pixel range. \item (D5) Reconstruct the image channels using the reconstructed blocks. For the image with padding during encoding process, crop each channel to the original size. \item (D6) Up-sample the chromatic channels. \item (D7) Convert the YUV image back to the RGB. 
~\vspace{-.75em}
\end{itemize}

\subsection{Basic Notations and Definitions}

Let $\Z$ and $\R$ denote the sets of integers and real numbers, respectively. For $x \in \R$, the {\it truncation} of $x$ to the range $[a, b] \subset \R$ is defined as: \((x)^a_b := \min(\max(x, a), b),\) which identifies the "closest" value within $[a, b]$ to \(x\). Equivalently, this operation can be expressed as:
\((x)^a_b = a\) if \(x \le a\); \((x)^a_b = x\) if \(a \le x \le b\); and \((x)^a_b = b\) if \(x \ge b\).
The {\it rounding} of \(x\), denoted by \([x]\), is the "closest" integer to \(x\). Formally: \([x] := \argmin_{z \in \Z} |z - x|,\) where \([x]\) is uniquely defined. If \(x\) is equidistant between two consecutive integers \(j\) and \(j+1\), a tie-breaking rule (such as rounding up or rounding down) is applied to ensure a unique solution. For \(q \in \Z\), the \(q\)-{\it quantization} operation is defined as:\([x]_q := q \left\lfloor \frac{x}{q} + 0.5 \right\rfloor,\)
or equivalently, \([x]_q:= q\left[{x \over q}\right],\) where the result is always an integer that is a multiple of \(q\).

With a slight abuse of notation, the truncating and rounding operations can be trivially extended to vectors, where the respective scalar transformations are applied element-wise to all components of a vector $\x \in \R^n$. Similarly, the quantization operation can be defined for vectors with a different quantization factor for each dimension, given by another vector $\q \in \Z^d$. This is expressed as:
$[\x]_{\q} = [\x \oslash \q] \odot \q$,
where $\oslash$ and $\odot$ denote element-wise division and multiplication (Hadamard operations), respectively. Furthermore, the truncating, rounding, and quantizing operations can be overloaded to apply to sets of vectors. For a set $A \subset \R^d$, these operations are defined as follows:
$(A)^b_a = \{({\mathbf a})_a^b \mid {\mathbf a} \in A\}$, 
$[A] = \{[{\mathbf a}] \mid {\mathbf a} \in A\}$, 
$[A]_{\q} = \{[{\mathbf a}]_{\q} \mid {\mathbf a} \in A\}$.
From these definitions, it is straightforward to observe that $\Z = [\R]$ and $\Z^d = [\R^d]$. Additionally, the set $[\Z]_q$ corresponds to the integers that are multiples of $q$.

The three operations can be viewed as projections. It is not hard to show that $(\x)_b^a = \argmin_{\z \in (\R^d)_b^a} \|\z - \x \|$, $[\x] = \argmin_{\z \in \Z^d} \|\z - \x \|$, and $[\x]_{\q} = \argmin_{\z \in [\Z^d]_{\q}} \|\z - \x \|$, where $\|\cdot\|$ denotes the $\ell_2$ (Euclidean) norm\footnote{This definition generalizes to any $p$-norm with $p \geq 1$.}. In addition, we have $([\x])_{b}^{a} = \argmin_{\z \in (\Z^d)_a^{b}} \|\z - \x \|$. The minima in these definitions can be made unique by introducing symmetry-breaking measures, as in the rounding operation.

\subsection{DCT and IDCT Matrices}
\label{sec:dct}

The most important steps in the JPEG compression and decompression process are DCT and IDCT, which are linear transforms applied to $8 \times 8$ pixel blocks. The $8$-point DCT can be represented as an $8 \times 8$ orthonormal matrix, given by:
$P_{kn} = {1 \over 2} \cos\left({\pi \over 8} (n+1/2)k\right)$.
The corresponding IDCT is represented by the transpose of this matrix, satisfying $P^\top P = P P^\top = I$.

For an $8 \times 8$, $8$-bit grayscale image block $X \in (\Z^{8 \times 8})_0^{255}$, represented as an $8 \times 8$ matrix, its DCT and IDCT are computed as $PXP^\top$ and $P^\top \tilde{X}P$, respectively. JPEG compression with a quantization matrix $Q$ is defined as:
$\tilde{X} = [(D^\top X D) \oslash Q]$,
and the corresponding JPEG decompression is:
$(D([\tilde{X} \otimes Q]D^\top))_{0}^{255}$. A single round of JPEG compression and decompression applied to an $8 \times 8$ pixel block, referred to as the {\it JPEG transform} in this work, is given by:
$T_{Q}(X) = ([D [D^\top X D]_Q D^\top])_{0}^{255}$. {In practical implementations of JPEG compression, the normalizing factor may be split between the DCT and IDCT matrices to improve numerical stability and efficiency. However, this adjustment does not affect the JPEG transform.}

It is more convenient to use the vectorized form of the JPEG transform. Let $\x = \text{vec}(X)$ and $\q = \text{vec}(Q)$. Additionally, we define a $64 \times 64$ orthonormal matrix $D = P \otimes P^\top$, where $\otimes$ denotes the Kronecker product. Using these definitions, the JPEG transform can be written more concisely as:
\begin{equation}
T_{\q}(\x) = ([D^\top[D\x]_{\q}])_0^{255}.
\label{jpeg}
\end{equation}
Note that $T_{\q}$ is an integer transform, converting a vector of integers to another vector of integers.

\subsection{JPEG Fixed Points of $8\times 8$ Blocks}

We define $\I_0$ as the set of all $64$-D vectors corresponding to {\it vectorized} $8 \times 8$ grayscale image blocks with $8$-bit pixels (subsequently referred to as the {\it image block} for simplicity)\footnote{The actual value range of the pixel values is not important, so we can use the original values $\{0, \cdots, 255\}$ or the centered values $\{-128, \cdots, 0, \cdots, 127\}$.}. The output of the JPEG transform of the elements in $\I_0$ forms the set $\I_1$, defined as $\I_1 = \{T_{\q}(\z) \mid \z \in \I_0\}$. Repeating this process, we can further define $\I_{t+1} = \{T_{\q}(\z) \mid \z \in \I_t\}$. 

\begin{thm}
$\I_t$ is finite and $\I_0 \supseteq \I_1 \supseteq \cdots \I_{t-1} \supseteq \I_{t}$.  
\label{thm1}
\end{thm}

\begin{proof}
We prove this by induction. First, since the JPEG-transformed image block is another image block, we have $\I_t \subseteq \I_0$. Next, assume $\I_{t} \subseteq \I_{t-1}$. For $\x_{t+1} \in \I_{t+1}$, there exists $\x_t \in \I_t \subseteq \I_{t-1}$ such that $\x_{t+1} = T_{\q}(\x_t)$. Because $\x_t \in \I_{t-1}$, we have $\x_{t+1} \in \I_t$, therefore $\I_{t+1} \subseteq \I_t$. Finally, since $\I_0$ is a finite set\footnote{The total number of elements in $\I_0$ is $2^{2^{2^3}} \approx 1.15792089 \times 10^{77}$. For comparison, the estimated number of atoms on Earth is $1.33 \times 10^{50}$.}, all $\I_t$ are also finite sets.
\end{proof}

\begin{thm}
$\forall \x_0 \in \I_0$ and a sequence $\x_0 \rightarrow \x_1 \cdots \rightarrow \x_t$ with $\x_{t+1} = T_{\q}(\x_t)$, $\exists \tau \in \Z$ and $\tau < \infty$, such that for $t \geq \tau$, $\x_t = \x_{t+1}$.  
\label{thm2}
\end{thm}

\begin{proof}
We define
$\epsilon_t  =  \|D\x_t - [D\x_t]_{\q}\|= \|\x_t - D^{\top}[D\x_t]_{\q}\|$, 
as the quantization error of the DCT of $\x_t$, $D\x_t$, and
$\eta_{t+1}  =  \|D\x_{t+1} - [D\x_t]_{\q}\| = \|\x_{t+1} - D^\top[D\x_t]_{\q}\|$, 
which is the rounding and truncating error of the IDCT of the quantized DCT block, $D^\top[D\x_t]_{\q}$. By definition, we have $\x_{t+1} = ([D^\top[D\x_t]_{\q}])_0^{255}$. The additional equalities in the definitions result from the fact that $D$ is orthonormal, which leaves the $\ell_2$ distance unchanged.

We show $\eta_{t+1} \leq \epsilon_t$, which follows from
$\eta_{t+1} = \|\x_{t+1} - D^\top[D\x_t]_{\q}\| \leq \|\x_t - D^\top[D\x_t]_{\q}\| = \epsilon_t$,
where the inequality is due to the projection property of the rounding and truncation operations: $\x_{t+1} = ([D^\top [D \x_t]_{\q}])_0^{255} = \argmin_{\z \in \I_0} \|\z - D^\top [D \x_t]_{\q}\|$.

In addition, we have $\epsilon_{t+1} \leq \eta_{t+1}$, as
$\epsilon_{t+1} = \|D\x_{t+1} - [D\x_{t+1}]_{\q}\| \leq \|D\x_{t+1} - [D\x_t]_{\q}\| = \eta_{t+1}$,
where the inequality is due to the projection property of the quantization operation: $[D \x_{t+1}]_{\q} = \argmin_{\z \in [\I_0]_{\q}} \|D \x_{t+1} - \z\|$.
Combining these results, we show that $\epsilon_0 \geq \eta_1 \geq \epsilon_1 \geq \eta_2 \cdots \geq 0$ is a non-increasing sequence. 

Next, when $\eta_{t+1} = \epsilon_t$, we have
$\|\x_t - D^{\top}[D\x_t]_{\q}\| = \|\x_{t+1} - D^\top[D\x_t]_{\q}\| = \min_{\z \in \I_0} \|\z - D^\top[D\x_t]_{\q}\|$,
and by the uniqueness of the minimum, $\x_t = \x_{t+1}$. 

Finally, we show that the convergence occurs in a finite number of steps. Since $\epsilon_t$ and $\eta_t$ can take only a finite number of values, as they are differences between elements of two finite sets, the non-increasing sequence can only change values finitely many times. Hence, equality will eventually occur within a finite number of iterations.
\end{proof}

Theorem \ref{thm2} is validated with numerical experiments based on $1$M random $8 \times 8$ image blocks of $8$-bit pixels, as shown in Fig.~\ref{fig:differences}. The plots illustrate $\epsilon_t - \eta_{t+1}$ (blue) and $\epsilon_{t+1} - \eta_{t+1}$ (red). Note that $\epsilon_t - \eta_{t+1}$ remains positive until convergence to zero, while $\epsilon_{t+1} - \eta_{t+1}$ stays below zero until convergence to zero. These observations are consistent with the theoretical prediction that $\epsilon_t \geq \eta_{t+1} \geq \epsilon_{t+1}$.

\begin{figure}
    \centering
    \includegraphics[width=.9\linewidth]{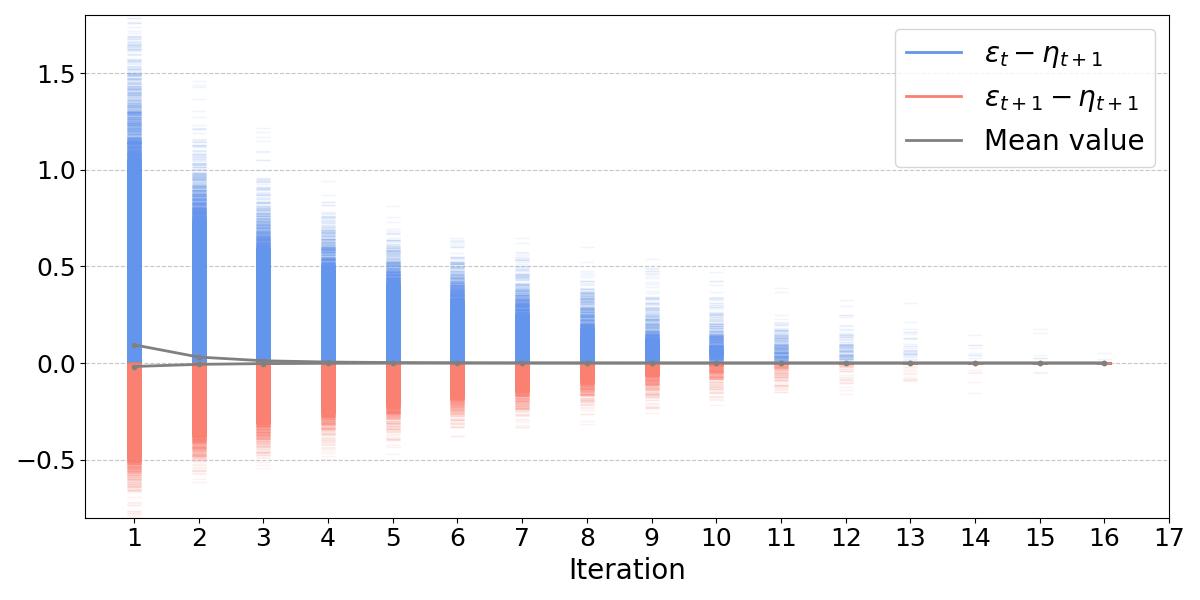}
    ~\vspace{-1.7em}
    \caption{\it \small {The distribution of $\epsilon_t-\eta_{t+1}$ and  $\epsilon_{t+1}-\eta_{t+1}$ from 1M $8\times 8$ $8$-bit image blocks. See details in texts.}}
    \label{fig:differences}
    ~\vspace{-2em}
\end{figure}
Theorems \ref{thm1} and \ref{thm2} are the foundation of the following theorem, which is the main theoretical result of this work.
\begin{thm}
The JPEG transform defined in Eq.\eqref{jpeg} has {\it fixed points}, which is defined as $\x \in \I_0$ such that $\x = T_{\q}(\x)$. All fixed points form a set $\I^{\star} = \bigcap_{t=0}^{\infty} \I_t$. Furthermore, $\forall \x_0 \in \I_0$, the sequence starting at $\x_0$, with $\x_{t+1} = T_{\q}(\x_{t})$, converges to a fixed point $\x^\star \in \I^\star$ in finite steps. Equivalently, $\exists t < \infty$, $\I^\star = \I_t$.   \label{thm3} 
\end{thm}
\begin{proof}
First, we show that when we have $\x_t = \x_{t+1}$, then all $\x_{t+k} = \x_t$ for $k \in \Z$. It suffices to show $\x_t = \x_{t+1} \rightarrow \x_{t+1} = \x_{t+2}$. Note that $\x_{t+2} = \argmin_{\z \in \I_0}\|\z - D^\top[D\x_{t+1}]_{\q}\|$. The RHS equals to $\argmin_{\z \in \I_0}\|\z - D^\top[D\x_{t}]_{\q}\|$ because we have $\x_t = \x_{t+1}$. This is just the definition of $\x_{t+1}$. Hence, we prove $\x_{t+1} = \x_{t+2}$. Then, subsequent $\x_{t+k}$ values will all be the same by induction. When that occurs, $\x_t = \x_{t+1} = T_{\q}(\x_t)$, so $\x_t$ is a fixed point. Using Theorem \ref{thm2}, we know that a fixed point can be reached in finite step from any $\x_0 \in \I_0$. Take the maximum number of steps for any $\x_0$ to reach a fixed point as t, we have $\I^\star = \I_t$.  
\end{proof}
The convergence of the JPEG fixed point described in Theorem \ref{thm3} is demonstrated in Figure \ref{fig:delta_patch} using the same set of $1$M image blocks with the graph of $\ell_2$ differences between the reconstructed patches of consecutive steps $\delta_t = \|\x_t - \x_{t+1}\|$ with different compression qualities (\ie, different quantization matrices). As these results show, the differences in consecutive image patches after repeated JPEG transforms continue to reduce and eventually diminish when a fixed point has been reached. 
%
\begin{figure}
    \centering
    \includegraphics[width=.9\linewidth]{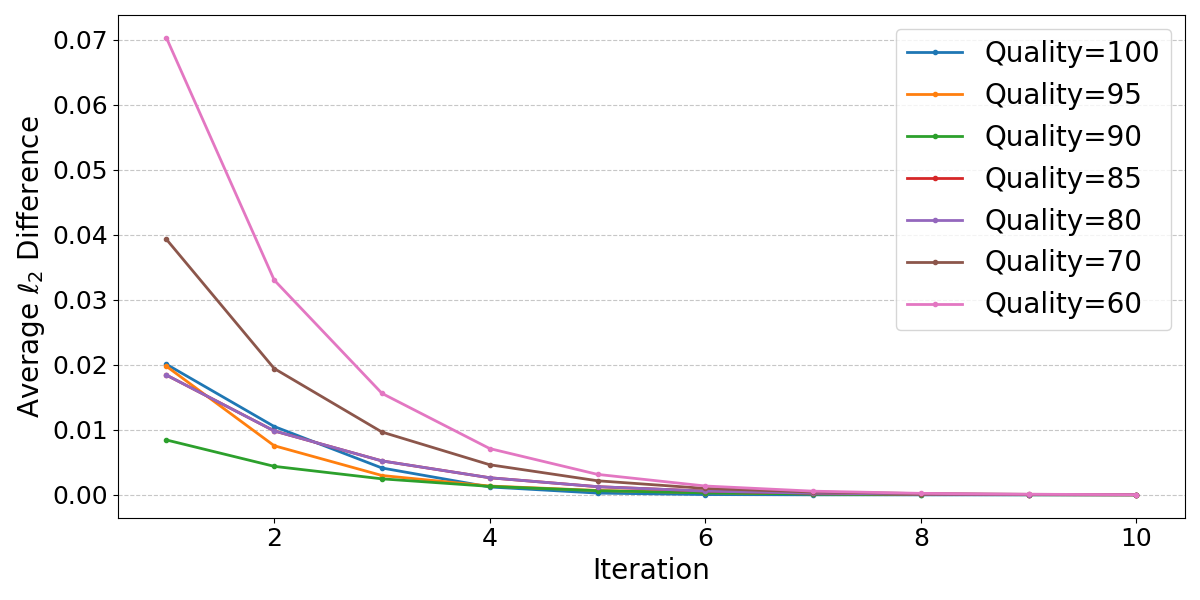}
    ~\vspace{-1.7em}
    \caption{\it \small \zf{The $\ell_2$ difference between consecutive patches for 1M $8 \times 8$ patches with various compression quality. The starting point of the process of finding fixed point is single JPEG compressed patch.}}
    \label{fig:delta_patch}
    ~\vspace{-2em}
\end{figure}
For $\forall \x \in \I_0$, let $\x^\star$ denote its corresponding fixed point of $T_{\q}$ and define $\delta = \max_{\x \in \I_0}\|\x - \x^\star\|$, \ie, it is the upper-bound of any image block to its corresponding JPEG fixed point. The following theorem shows that image blocks sufficiently different from each other will converge to distinct JPEG fixed points.  
\begin{thm}
For $\x, \x' \in \I_1$ (\ie, they are already in JPEG formats) with their corresponding JPEG fixed points $\x^\star$ and $\x'^\star$, if $2\delta < \|\x - \x'\|$, then $\x'^\star \neq \x^\star$.
\label{thm4}
\end{thm}
\begin{proof}
We prove the result by showing a contradiction. Assume that $\x^\star = \x'^\star$, then by definition of $\delta$, we have $\|\x - \x^\star\| + \|\x' - \x^\star\| \le 2\delta$.

On the other hand, with the triangle inequality, we also have $\|\x - \x'\| \le \|\x - \x^\star\| + \|\x' - \x^\star\|$,
which implies $2\delta < \|\x - \x'\| \le \|\x - \x^\star\| + \|\x' - \x^\star\| \le 2\delta$. This is a contradiction. Hence, we have proven the result.
\end{proof}

\begin{figure}
    \centering
    \includegraphics[width=.9\linewidth]{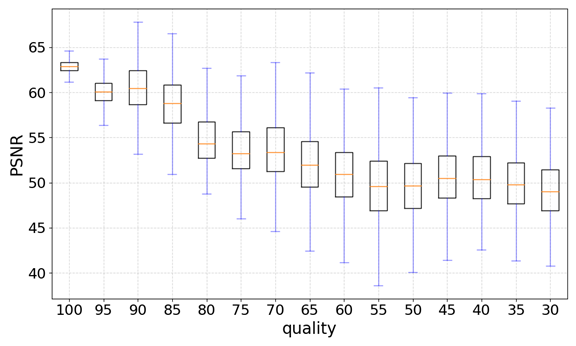}
    ~\vspace{-1em}
    \caption{\it \small {The PSNR between fixed-point image and single JPEG image for different compression quality. }}
    \label{fig:qualaity}
    ~\vspace{-2em}
\end{figure}

\subsection{JPEG Fixed Points of Full-sized Images}

We have established the existence and convergence of fixed points for $8 \times 8$ blocks. For a full grayscale image, each $8 \times 8$ block is processed independently of the others. The JPEG fixed point image is obtained by combining the fixed points of each block. To avoid the need for padding, we crop the image size to a multiple of $8$.

Theorem \ref{thm4} and empirical evaluations in Fig.~\ref{fig:qualaity} illustrate the quality degradation introduced by the JPEG fixed point image compared with an original JPEG image with the same quantization. The results confirm that the quality degradation introduced by the JPEG fixed points is typically less than $25$dB below the threshold of human visual perception.

For RGB images, the color space transform is usually implemented as a one-to-one lookup table. The upsampling and downsampling of the color channels do not interfere with the fixed-point computation because the pixels being downsampled and upsampled do not participate in subsequent computations (the same set of pixels will be downsampled and upsampled without interacting with other pixels).

\section{Tamper-evident JPEG Images}
\label{sec:tamper}
\begin{figure*}[t]
    \centering
    \includegraphics[width=.9\linewidth]{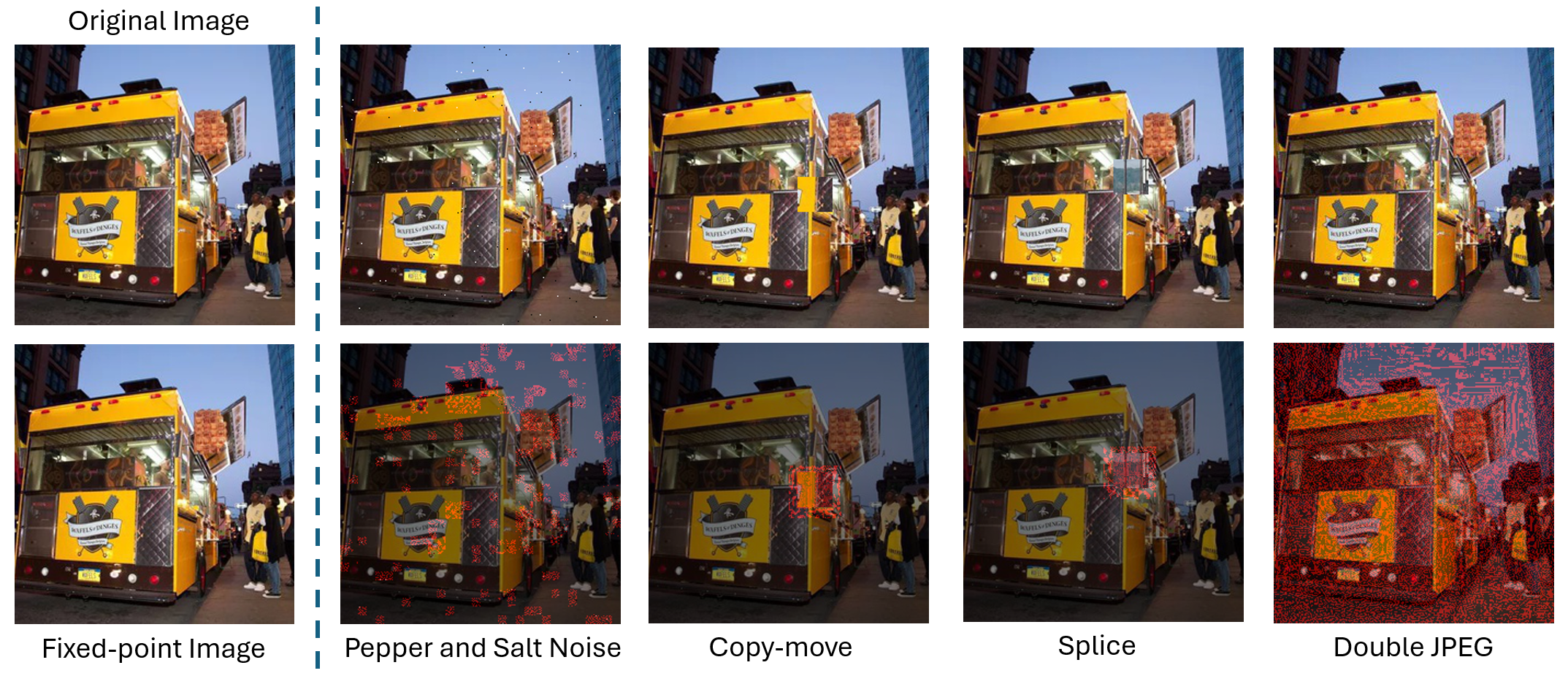}
    ~\vspace{-1em}
    \caption{\it \small Example of fixed points of RGB images and detection and localization of tempering. The images are better viewed in color.
    }
    \label{fig:example_rgb}
    ~\vspace{-2.5em}
\end{figure*}

Digital images comprise more than 90\% of the information we consume online. Digital images have been subject to edits since their inception, ranging from harmless enhancements to deliberate manipulations with malicious intent. It is vital to develop methods to verify the integrity of images to ensure our trust in them.

There are two main strains of image integrity verification methods: watermarking and fingerprinting (controlled capture). {\it Watermarking} entails embedding unique, specially designed signals directly into the image. These signals serve as identifiers, enabling authenticity verification and the detection of alterations. {\it Fingerprinting}, on the other hand, extracts distinctive features or patterns from an image that uniquely identify the original content. These fingerprints are stored and later used to compare with extracted fingerprints from an image to determine whether it undergoes any modifications.

Tamper-evident images \cite{li2014image} present an alternative to watermarking and fingerprinting. A tamper-evident image is its own proof of authenticity—any modification to the image induces detectable changes, making tampering evident. JPEG fixed points provide a means of creating tamper-evident images. Such an image remains unchanged with additional JPEG compression and decompression as long as the same quantization table is used. Any modification or tampering of the image disrupts these fixed points, which can be detected by reapplying the JPEG transform and computing the difference between the block before and after the transformation. Furthermore, identifying blocks with non-zero differences allows for the localization of tampered regions.

The procedure for creating a tamper-evident JPEG image is simple. Starting with the original quantization table, we repeatedly apply JPEG compression and decompression until the JPEG fixed-point image is reached. 
{Fig. \ref{fig:example_rgb} demonstrates the tamper-evident JPEG images. We apply four different operations to the tamper-evident image: 1) paper and salt noise, 2) copy-move, 3) splice, and 4) double JPEG with a different quantization table. These manipulated images are then subjected to an additional JPEG compression, employing the same quantization table as the tamper-evident image. The changed blocks are highlighted in red and overlaid on the manipulated images as masks, shown in the second row of Fig. \ref{fig:example_rgb}. For manipulations 1), 2), and 3), all blocks containing manipulated regions are identified, whereas for the double JPEG image, most pixels in the image differ from the tamper-evident image.}

\section{Related Works}
\label{sec:rel}

Tamper-evident images were first investigated in \cite{li2014image}, where the authors proposed using the fixed points of a Gaussian convolution and deconvolution to create tamper-evident images. However, the theoretical justification for the existence of a fixed point is probabilistic, and for the approach to be practical, the size of the Gaussian convolution kernel must be chosen to be small. Moreover, such a scheme is unlikely to survive JPEG compression, which was not considered in the original work.

As the dominant image format, JPEG's properties have been widely utilized to detect image manipulation. In particular, the pioneering work of \cite{popescu2004statistical} introduced a general approach to detect traces of double JPEG compression to reveal manipulations in an image. The rationale is that when a manipulated JPEG image is saved, it undergoes compression twice—first during the initial encoding and then again after modifications are made and the image is re-saved. This work has inspired a large body of subsequent research (\eg, \cite{verma2024block, niu2021detection, park2018double, huang2010detecting, chen2008machine}) to use double JPEG detection for image forensics.

In the study of double-JPEG detection with the same quantization matrix, JPEG's fixed point property was first observed in \cite{huang2010detecting}. However, to the best of our knowledge, there has not been a formal analysis to theoretically prove this property, which is achieved in our work.


\section{Conclusion}
\label{sec:cond}

In this work, we re-examine a long-standing and intriguing property of JPEG: the existence of fixed points that can be reached after repeatedly applying JPEG transforms. Using a rigorous mathematical formulation, we prove the existence and convergence of JPEG fixed points. Furthermore, we describe a method for creating tamper-evident JPEG images by leveraging the fixed point property.

In future work, we aim to expand the current theoretical results to RGB images, considering color space transforms, downsampling and upsampling steps, and boundary padding. Additionally, we plan to explore the potential applications of JPEG fixed points in detecting image manipulations.
 
\newpage
\bibliographystyle{IEEEbib}
\small
\bibliography{refs}

\end{document}